\documentclass{statsoc}
\usepackage[a4paper]{geometry}
\title[Poisson Binomial Models]{Some New Results for Poisson Binomial Models}

\usepackage{amsmath, amsfonts}
\usepackage{amsmath}
\usepackage{graphicx}
\usepackage{url}
\usepackage{etoolbox}

\author[Rosenman{\it et al.}]{Evan Rosenman}
\address{Stanford University,
Stanford,
         CA.}
\email{rosenman@stanford.edu}

\def\E{\mathbb{E}}

\newcommand{\pij}{p_{ij}}

\newcommand{\xij}{\bsx_{ij}}

\newcommand{\yij}{Y_{ij}}

\newcommand{\cov}{\text{Cov}}

\newcommand{\bsx}{\boldsymbol{x}}

\newcommand{\tran}{\mathsf{T}}

\newtheorem{theorem}{Theorem}

\begin{document}

\begin{abstract}
We consider a problem of ecological inference, in which individual-level covariates are known, but labeled data is available only at the aggregate level. The intended application is modeling voter preferences in elections. 

In \cite{rosenman2018using}, we proposed modeling individual voter probabilities via a logistic regression, and posing the problem as a maximum likelihood estimation for the parameter vector $\beta$. The likelihood is a Poisson binomial, the distribution of the sum of independent but not identically distributed Bernoulli variables, though we approximate it with a heteroscedastic Gaussian for computational efficiency. Here, we extend the prior work by proving results about the existence of the MLE and the curvature of this likelihood, which is not log-concave in general. We further demonstrate the utility of our method on a real data example. Using data on voters in Morris County, NJ, we demonstrate that our approach outperforms other ecological inference methods in predicting a related, but known outcome: whether an individual votes. 
\end{abstract}


\section{Introduction}

This paper considers a problem of ecological inference. We suppose independent binary variables $\yij$ follow a logistic regression model in which $p_{ij} := \Pr(\yij=1\mid \xij) = \sigma(\xij^\tran\beta)$ where $\sigma(z) = (1+\exp(-z))^{-1}$.
We would like to estimate $\beta$ from $(\xij,\yij)$ values but while $\xij$ values are available, the individual $\yij$
values are not.  All that are available are the totals $D_i =\sum_{j\in S_i}\yij$ for some disjoint sets $S_i$ of individuals.

This problem has a long history in political science beginning with early efforts to identify voting preferences by race from precinct level data \citep{sun2017probabilistic, flaxman2015supported}.  The modern setting has a rich
set of covariates $\xij$ for each voter $j$ within each precinct $i$, derived from the ``voter file," the roster of all registered voters within a given geography. Vote tallies are reported for each precinct $i$. The goal is to design voter mobilization and persuasion efforts for a subsequent election, using retrospective knowledge of whom each voter was likely to have supported in the prior election. 

Following more recent developments in the machine learning literature \citep{patrini2014almost, sun2017probabilistic}, we explicitly pose the problem as a maximum likelihood estimation. The exact probability of getting $D_i$ is 
\[   \sum_{A \in F_{D_i}} \prod_{j \in A} p_{ij} \prod_{j \in A^c}(1 - p_{ij})  \,, \] 
where $F_{D_i}$ is the set of all subsets of $S_i$ with cardinality $D_i$. This is called the Poisson binomial distribution \citep{wang1993number}. The exact log-likelihood evaluated at $\beta$ is then 
\begin{equation}\label{eq:lik}
\ell(\beta) = \sum_i \log \left( \sum_{A \in F_{D_i}} \prod_{j \in A} \sigma(\xij^\tran\beta) \prod_{j \in A^c}(1 - \sigma(\xij^\tran\beta)) \right)
\end{equation}
where the sets $F_{D_i}$ correspond to the observed vote totals $D_i$. These sets are combinatorially large and thus the likelihood is computationally intractable for even modestly sized problems. Various approximate inference techniques have been used with similar formulations \citep{sun2017probabilistic, jackson2006improving}. Our computationally favorable approach, first discussed in \cite{rosenman2018using}, is to approximate the Poisson binomial log likelihood by a heteroscedastic Gaussian one, justified by a Central Limit Theorem. Here, we extend our prior work by proving new theoretical results about the likelihood under this parameterization, and demonstrating that our method outperforms ompetitor ecological inference techniques on a real data set. 

An outline of this paper is as follows. Section \ref{sec:related} reviews prior work on ecological inference from the political science, statistics, and machine learning literatures. Section \ref{sec:theory} proves theoretical results about the existence of a finite MLE and the curvature of the log likelihood. The log likelihood is not guaranteed to be concave, so we cannot be sure that we have found the global optimum. We can, however, evaluate the resulting model fits on held-out precincts and compare it to other methods. Section \ref{sec:optimization} briefly reviews the motivation for our technique. In Section \ref{sec:morris}, we compare performance of our methods against competitor ecological inference models, including a single layer neural network. We use a data set of voters from Morris County, NJ, and build a classifier to predict whether an individual will vote in each election, using only aggregate turnout numbers as training data. This outcome is closely related to our desired outcome (voter preference rather than voter participation), but it is observed, allowing us to directly compare predictive accuracy against other models. Section \ref{sec:conc} contains our conclusions. Proofs can be found in the Appendix.

\section{Related Work}\label{sec:related}

Theoretical work on the Poisson binomial distribution has focused on computationally tractable ways to estimate its distribution function, often via approximations to other distributions \citep{EHM19917, roos1999, chen1974}. Prior research \citep{HONG201341} has identified a closed-form expression for the CDF, which relies on the discrete Fourier Transform. This technique is leveraged in the \texttt{poisbinom} package \citep{poisbinom}, which we use for this paper. The application of the Poisson binomial distribution to the generalized linear model setting has been discussed by \cite{Poibi}, who propose it for hypothesis testing on the parameter vector for a logistic regression model. 

Literature on the ecological inference problem bifurcates into two main subfields: political science and machine learning. The problem originates in the political science literature, where early work focused on inferring voting patterns by race \citep{sun2017probabilistic, flaxman2015supported}. Among the simplest techniques is Goodman's Regression \citep{goodman1953ecological}, dating to the middle of the twentieth century, in which vote proportions are regressed on the proportion of voters of a specific race in order to generate an individual-level model. More advanced methods, making use of random effects, were proposed at the end of the century. These included a semiparametric approach proposed by \cite{prentice1995aggregate} and several hierarchical models proposed by King \citep[see][]{king1997solution, king1999binomial}. 

Wakefield made a number of noteworthy contributions, including positing a statistical framework for ecological inference \citep{wakefield2001statistical} and showing that if covariate totals are known and conditioned upon in the 2 $\times$ 2 table, this case yields a convolution of binomial likelihoods \citep{wakefield2004ecological}. From the latter insight, he developed a normal approximation for efficient inference. Jackson et al. generalized much of the prior work with their integrated ecological model \citep{ecoreg, jackson2006improving}, in which the individual-level probabilities are modeled via a logistic regression and then averaged over the population in each area. The count of votes for a particular candidate in the area is then modeled as a binomial random variable with success probability equal to this average. Our approach is similar to that of Jackson et al., except that we model the data as Poisson Binomial rather than Binomial, which means probabilities need not be averaged. The Jackson method is implemented in the \texttt{ecoreg} \citep{ecoreg} package, against which we baseline in Section \ref{sec:morris}.

The machine learning literature is more varied, both in methodology and in application. For the problem of inferring voting behavior in demographic subgroups, distribution regression is a popular tool. This method maps the covariate distribution within each geography to a single high-dimensional covariate vector via kernel mean embeddings; penalized or Bayesian regressions methods are then used to fit a function mapping from these distributions to observed vote proportions \citep{law2017bayesian}. Distribution regression has been deployed by Flaxman for analysis of subgroup preferences in the 2012 \citep{flaxman2015supported} and 2016 \citep{flaxman2016understanding} elections. Yet, because it aggregates over individual-level data, the method is more appropriate for understanding group-level behavior.

SVM-based methods have also gained traction. Rueping et al. introduced an ``inverse calibration" method \citep{rueping2010svm}, which a Support Vector Regression is fit to the data such that the average prediction within each geography is close to the sigmoid inverse of the bag probability. Yu et al. also work in the large-margin framework, proposing a ``$\propto$SVM" method in which a loss is directly minimized over the model parameters and the unseen individual labels \citep{felix2013psvm}. These methods may be more appropriate for generating individual-level classifiers.

A last, loosely connected area of the literature might be termed ``learning with label proportions." These papers are focused specifically on learning individual-level classifiers, and, relative to distribution regression and SVM-based methods, engage more directly with a probabilistic model for the data. K{\"u}ck and de Freitas proposed a hierarchical probabilistic model and an MCMC algorithm for training, and showed their method was effective in the object recognition domain \citep{kuck2012learning}. Quadrianto et al. introduced the mean map algorithm \citep{quadrianto2009estimating}, in which models are fit by maximizing the log-likelihood in a conditional exponential family model. Their method requires a somewhat restrictive assumption that the distribution of the covariates is independent of geography conditional on the vote proportions. Patrini et. al were able to generalize this work and weaken this assumption \citep{patrini2014almost}. They proposed several novel algorithms: Laplacian Mean Map (LMM) and Alternating Mean Map (AMM). Sun et. al used related techniques for analysis of the 2016 presidential election, defining a likelihood and maximizing it via a novel exact inference algorithm making use of counting potentials \citep{sun2017probabilistic}. 

Among the prior work, our approach is most similar to that of Sun, though we propose a different set of algorithms for model fitting. Compared to the broader literature, we differ in a few key ways. Our algorithms are designed solely for the case in which covariates are known for all participating voters in a geographic area, and do not generalize to the case of partial samples via, for example, data from the American Community Survey. We also make relatively strong assumptions on the functions relating individual-level covariates to aggregate statistics. Moreover, our method is purely frequentist while much of the literature uses Bayesian methods. The benefits of our proposal include simpler fitting procedures, straightforward estimation of individual-level probabilities, and greater model interpretability.

\section{Theoretical Results}\label{sec:theory}

We seek to maximize the likelihood over $\beta$. We first consider several properties of the likelihood. 

\subsection{Existence of a Finite MLE}

In the case of standard logistic regression, it is well-known that the MLE may not exist \citep[see e.g.][]{candes2018phase}. Our scenario is somewhat more delicate. We begin by considering the scenario in which an MLE fails to exist. 

We can write the log likelihood as 
\begin{equation}\label{eq:likForm}
\ell(\beta) = \sum_i \log \left( \sum_{A \in F_{D_i}} \exp \left( \sum_{j \in A} \xij^T \beta \right) \right) - \sum_{j \in S_i} \log \left( 1 + \exp(\xij^T \beta) \right) \,.
\end{equation}
Suppose we can find a direction $\tilde \beta$ and sets $A_1^{\star} \in F_{D_1}, A_2^{\star} \in F_{D_2}, \dots, A_n^{\star} \in F_{D_n}$, where $n$ is the number of precincts, such that 
\[ \tilde \beta^T \xij > 0 \text{  for all $\xij \in \bigcup_{i = 1}^n A_i^{\star}$} \hspace{3mm} \text{ and } \hspace{3mm}  \tilde \beta^T \xij < 0 \text{  for all $\xij \in \bigcup_{i = 1}^n (A_i^{\star})^c$} \] 
For each precinct, define $H_i = \sum_{A \neq A_i^{\star} \in F_{D_i}} \exp \left( \sum_{j \in A} \xij^T \beta \right) $. We can now rewrite the log-likelihood evaluated at $\tilde \beta$ as: 
\[ \ell(\tilde \beta) = \sum_i \log \left(H_i +   \sum_{j \in A_i^{\star}}  \exp \left( \xij^{\tran} \tilde \beta \right)  \right) - \sum_{j \in S_i} \log \left( 1 + \exp \left(\xij^{\tran} \tilde \beta \right) \right)   \] 
If we increase the magnitude of $\tilde \beta$, then for each $i$ we will make $\sum_{j \in A_i^{\star}}  \exp \left( \xij^{\tran} \tilde \beta \right)$ arbitrarily large and $\sum_{j \in (A_i^{\star})^c}  \exp \left( \xij^{\tran} \tilde \beta \right)$ arbitrarily close to 0; the former term will also dominate $H_i$. It follows that, for each precinct, the first term  behaves as $\log \left( \sum_{j \in A_i^{\star}}  \exp \left( \xij^{\tran} \tilde \beta \right)  \right)$ at large $||\tilde \beta||_2^2$, while the second term behaves as $-\log \left( \sum_{j \in A_i^{\star}}  \exp \left( \xij^{\tran} \tilde \beta \right)  \right)$. Thus, making $\tilde \beta$ arbitrarily large will yield a log-likelihood arbitrarily close to 0, and no finite MLE will exist. In words, $\beta$ is normal to a hyperplane that perfectly separates $D_i$ units from $|S_i| - D_i$ units for all precincts $i = 1, \dots,  n$. 

These insights are formalized as follows. 

\begin{theorem}\label{thm:finMLE}
For every set of individuals $A$, define $K_A$ as the dual cone of $A$, 
\[ K_A = \{\beta \mid \beta^T \xij \geq 0 \text{ for all } \xij \in A\} \,,\]
and $K'_A$ as the polar cone of $A$: 
\[ K'_A = \{\beta \mid \beta^T \xij \leq 0 \text{ for all } \xij \in A\} \]
If 
\[ \bigcap_{i=1}^n \bigcup_{A \in F_{D_i}} \text{int}(K_A) \cap \text{int}(K'_{A^c}) \neq \emptyset \]  
then there does not exist a finite MLE. 
\end{theorem}
The proof follows directly from the preceding argument. Unions and intersections of cones are themselves cones, so the set of potentially infinite norm $\beta$ values form a cone. 

Theorem \ref{thm:finMLE} is potentially troublesome for scenarios with small values of $n$. Directions for $\beta$ satisfying the perfect separation conditions can be readily found in simulations with fewer than 10 precincts. But as $n$ grows, it becomes increasingly likely that the problematic $\beta$ set is empty, due to the outer intersection. In settings in which we are interested -- typically involving at least a few hundred precincts -- we thus assume that a finite MLE exists. Deriving exact conditions for the existence of a finite MLE remains an open question for further research. 


\subsection{Curvature Results}

We prove some elementary results regarding curvature of the likelihood function. 

\begin{theorem}
The log likelihood is a difference of convex functions.
\end{theorem}

\begin{proof}
With the log likelihood given by 
\begin{align*}
\ell(\beta) = \sum_i \log \left( \sum_{A \in F_{D_i}} \exp \left( \sum_{j \in A} \xij^T \beta \right) \right) - \sum_{j \in S_i} \log \left( 1 + \exp(\xij^T \beta) \right) \,,
\end{align*}
we see that the first term is a log-sum-exp function of $\beta$ in canonical form, while the second term is a sum of log-sum-exp functions of $\beta$. It follows immediately that both the first and second terms are convex functions in $\beta$ and hence that the log-likelihood is a difference of convex functions \citep{boyd2004convex}. 
\end{proof}

DC functions are neither convex nor concave in general, and one can readily generate examples where the Hessian of the log-likelihood has positive and negative eigenvalues.


%
\begin{theorem}
Suppose the model for the data is correct with true parameter value $\beta^{\star}$, and suppose the $|\xij|$ are bounded. As the number of precincts goes to infinity, the likelihood is asymptotically log-concave at the true parameter value $\beta^{\star}$.
\end{theorem}
\begin{proof}
Per the results in Appendix \ref{app:gradComp}, we can write the scaled Hessian as: 
\[ \frac{1}{n} \nabla^2 \ell(\beta) = \frac{1}{n} \sum_i \cov\left( \sum_{j \in S_i} \xij Y_{ij} \mid \sum_{j \in S_i} Y_{ij} = D_i \right) - \cov \left( \sum_{j \in S_i} \xij Y_{ij} \right) \] 
where $n$ is the number of precincts. By Kolmogorov's Strong Law \citep{sedor2015law}, we see 
\begin{align*}
&\frac{1}{n} \sum_i \cov\left( \sum_{j \in S_i} \xij Y_{ij} \mid \sum_{j \in S_i} Y_{ij} = D_i \right) \stackrel{a.s.} \to \frac{1}{n} \sum_i  \E \left( \cov\left( \sum_{j \in S_i} \xij Y_{ij} \mid \sum_{j \in S_i} Y_{ij} = D_i \right) \right) \\
&\hspace{33mm} =  \frac{1}{n} \sum_i   \cov \left( \sum_{j \in S_i} \xij Y_{ij} \right) - \cov \left( \E\left( \sum_{j \in S_i} \xij Y_{ij} \mid \sum_{j \in S_i} Y_{ij} = D_i \right) \right) 
\end{align*}
where the second line is due to the Law of Total Covariance. Thus
\[ \frac{1}{n} \nabla^2 \ell(\beta) \stackrel{a.s.} \to - \frac{1}{n} \sum_i \cov \left( \E\left( \sum_{j \in S_i} \xij Y_{ij} \mid \sum_{j \in S_i} Y_{ij} = D_i \right) \right) \,.\]
and the result follows from the fact that any covariance matrix must be positive semidefinite. 
\end{proof}

The likelihood will be asymptotically log-concave in a neighborhood around the true parameter value if there is sufficient differentiation in the covariates $\xij$ across precincts. To see why, consider the quantity 
\begin{align*}
\frac{1}{n} \sum_i \E\left( \sum_{j \in S_i} \xij (Y_{ij}-p_{ij}) \mid \sum_{j \in S_i} Y_{ij} = D_i \right)\E\left( \sum_{j \in S_i} \xij (Y_{ij}-p_{ij}) \mid \sum_{j \in S_i} Y_{ij} = D_i \right)^\tran\,.
\end{align*}
Under an identical SLLN argument, this quantity has the same asymptotic limit as $\frac{1}{n} \nabla^2$. Being an average of outer products, the resultant matrix will also have full rank as long as the vectors 
\[ v_i : = \E\left( \sum_{j \in S_i} \xij (Y_{ij}-p_{ij}) \mid \sum_{j \in S_i} Y_{ij} = D_i  \right) \]
span $\mathbb{R}^p$, where $\xij \in \mathbb{R}^p$. In practical examples, this condition holds as long $n$ is sufficiently large compared to $p$, and the covariates differ sufficiently across precincts. 

If the $v_i$ indeed span $\mathbb{R}^p$, then the Hessian has full rank and is thus negative definite, not merely negative semidefinite. Since the eigenvalues are continuous functions of the $\xij$, it follows that we are guaranteed log-concavity in a neighborhood around $\beta^{\star}$.


\section{Optimization}\label{sec:optimization}

Results from the prior section indicate that algorithms dependent on concavity, like Newton's Method, may not converge in the general case. But for large problems with substantial precinct differentiation, we may be able to exploit curvature near the true parameter value. 

A second consideration is runtime. The work of \cite{HONG201341} has substantially sped up the computation of Poisson binomial probabilities, but these computations can still be a bottleneck for even modestly sized problems. The exact gradient can be calculated by computing a number of Poisson binomial probabilities linear in the dataset size, per the results in Appendix \ref{app:gradComp}. But in practice, even this method is slow to train in the regimes of interest. 

These considerations point toward our approach, discussed in detail in \cite{rosenman2018using}. The Lyapunov CLT \citep{billingsley1995probability} can be used to observe that the asymptotic distribution of $D_i$ is given by
\[ D_i \stackrel{d} \longrightarrow N \left(\sum_{j\in S_i} p_{ij}, \sum_{j \in S_i} p_{ij}(1-p_{ij}) \right) \,. \] 
Thus, the contribution of precinct $i$ to the overall log likelihood is approximately
\[ \ell_i(\beta) \approx -\log \left( \phi_i\right) + \frac{1}{\phi_i^2} \left( D_i -\mu_i\right)^2,\]
where irrelevant constants have been dropped, $\mu_i =  \sum_{j \in S_i} p_{ij}, \phi_i^2 = \sum_{j \in S_i} p_{ij}(1-p_{ij})$, and $p_{ij} = \sigma(\xij^\tran\beta)$. 

This yields a gradient of the form: 
\begin{align*}
\nabla_{\beta} \ell_i \approx& \frac{1}{\phi^2} (D_i - \mu_i) \left( \sum_j \pij (1 - \pij) \xij \right)-\\&  \frac{1}{2} \left(\frac{(D_i - \mu_i)^2}{\phi_i^4} - \frac{1}{\phi_i^2} \right) \left( \sum_{i} (2 \pij - 1)(1-\pij)\pij \xij \right) \,.
 \end{align*}
These approximate gradients can be used in gradient ascent with a fixed step size. Or they can be used to choose an ascent direction with step size selected based on the true or approximate likelihood evaluated at each point along a grid. These are the primary algorithms we propose. 

\section{Comparison Against Other Ecological Inference Methods}\label{sec:morris}

Our goal with these models is to predict the candidate selected by each voter---an outcome that is unknown. This poses a challenge for evaluating the performance of our models and comparing against other ecological inference methods. 

To address this issue, we use a related task: modeling the probability that an individual casts a ballot (rather than modeling the candidate he or she supports). We train only on aggregated ballot counts from each precinct. This task is not a perfect proxy for modeling candidate selections, but it is an attractive option because it allows us to leverage the same set of covariates and the same aggregation structure, and we also have access to the individual-level outcomes for performance evaluation. 

We use a data set comprising all voters from Morris County, New Jersey, an affluent and historically Republican-leaning county of about half a million residents. The voter file contains 316,724 registered voters and includes limited demographic information as well as information about whether each voter cast a ballot in general elections and primaries stretching back to the year 2000. There are 396 voting precincts in the county. 

We fit models to eight data sets in total. To explore performance in different outcome regimes, we predict whether voters participated in each election from 2014 to 2017, in which 34\%, 19\%, 76\%, and 45\% of all voters in our data set cast a ballot, respectively. For each year, we fit two models: a parsimonious ``demographics-only" model containing just four covariates (age, party, gender, and whether the voter lives in an apartment); and a ``demographics and voter history" model that also contains nine variables corresponding to the voter's participation and voting method in the given year's primary and the primaries and general elections of the prior four years. 

We compare performance of a number of methods:
\begin{itemize}
\item To obtain an upper bound on performance, we fit a logistic regression and a Gradient Boosted Machine (GBM) to the data set while giving them access to the individual-level outcomes \citep{friedman2001greedy}. Because these models ``see" individual-level data, they should outperform methods that only have access to aggregated data. 
\item We fit three variants of our logistic regression formulation. 
\begin{itemize}
\item In the first (``Logit with Gaussian Gradient"), the coefficients are fit via gradient ascent exclusively using the Gaussian approximation to the gradient proved in Appendix \ref{app:lyapunovProof}. We run for 120 iterations using a learning rate of $2 \times 10^{-5}$. 
\item In the second (``Logit with Gaussian Gradient, PoiBin Backtracking"), we run ten iterations using the approximated gradient and fixed step size; for the remaining 110 iterations, we use the normal-approximation gradient to choose an ascent direction, but use backtracking line search based on the \emph{true} likelihood to choose a step size.
\item The third algorithm (``Logit with Gaussian Gradient, PoiBin Backtracking, True Gradient") is identical to the second, except we run only 100 iterations using backtracking line search. The final ten iterations are then instead run using the true gradient derived in Appendix \ref{app:gradComp}. 
\end{itemize}
These three variants are used to explore the practical effect of the Gaussian approximation on our model's accuracy. 
\item We also fit a neural network in which individual success probabilities are modeled as a feedforward network with a single hidden layer. Full details on the model and its gradients can be found in Appendix \ref{sec:neural}. Results were found to be highly sensitive to the parameter initializations and training lengths; hence, ten random initializations were tried, and parameters were stored after 50, 100, 150, and 200 iterations of gradient ascent. A development set consisting of 40 precincts (10\% of the total) was used to tune these hyperparameters. As true outcomes would not be known in practice, tuning was done by choosing the parameter that minimized the sum squared error between the predicted and actual aggregate voting rate in the development counties. 
Results were less sensitive to hyperparameters like the learning rate and number of hidden nodes, so these were set to $2 \times 10^{-6}$ and $10$ respectively. 
\item Following the approach in \citep{rueping2010svm}, we  baseline against the simplest ecological inference method: assigning each unit in a given aggregation block the average of the outcomes in that aggregation block. In our setting, this means each voter in a precinct is assigned the voter turnout proportion in that precinct as a pseudo-outcome, and a logistic regression model is fit to these data.  
\item We baseline against ecological regression as implemented in the \texttt{ecoreg} package in R \citep{ecoreg}. 
\item We baseline against Rueping's Inverse Calibration method \citep{rueping2010svm}. Aggregate accuracy on the 40 precincts in the development set is again used, this time for tuning the $C$ and $\epsilon$ parameters. 
\item Lastly, we baseline against the Mean Map \citep{quadrianto2009estimating}, Laplacian Mean Map, and Alternating Mean Map \citep{patrini2014almost}. Hyperparameters are again tuned using squared error on the development set. 
\end{itemize}

Results for the demographics-only model are provided in Table \ref{tab:demogModel} and results from the expanded data set are provided in Table \ref{tab:allModel}. The relative strength of the logistic regression formulation is immediately obvious: these models achieve the highest ROC AUC values in all but one of the eight conditions, and frequently come very close to the performance of methods with access to the individual outcomes. Also evident is the fact that little to no predictive power is gained by making use of the real likelihood rather than the approximation. Backtracking on the true Poisson Binomial likelihood or using the true gradient actually slightly degrade performance in most cases, while also slowing training. 

The neural network model is never competitive. It appears that the model is able to capture some useful interactions and non-linearities when more covariates are present, but it is generally too expressive and tends to overfit. 

Ecological regression performs well in all conditions, and outperforms our proposed methods in the demographics-only model for 2014. The other tested methods are generally not competitive. The logistic regression on aggregates technique performs surprisingly well given its extreme simplicity, but it still underperforms the proposed logistic methods. Inverse calibration sees a noticeable performance bump with the inclusion of additional covariates. The Mean Map, LMM, and AMM methods typically do poorly, with only AMM consistently beating random guessing in the demographics-only case. Each of these methods is somewhat sensitive to hyperparameter values, and tuning is extremely challenging in the absence of labeled data in the development set. We are using squared error across development precincts as a proxy measure, and it's highly plausible that alternative proxies would yield better hyperparameter values. Nonetheless, a strength of our proposed methods is that they require very little tuning in order to get good performance. 

\begin{table}
\caption{\label{tab:demogModel}ROC AUC scores for models predicting voter turnout, fit to demographics-only data sets. Highest values among ecological inference models are underlined.}
\centering
\fbox{%
\begin{tabular}{lllll}
                                                   & \multicolumn{4}{c}{\textbf{Demographics Only}} \\
                                                   & \textbf{2017}       & \textbf{2016}      & \textbf{2015}      & \textbf{2014}      \\
\textbf{Standard Methods (non-ecological)}         &            &           &           &           \\
\hspace{5mm}Logistic Regression                                & 72.0\%     & 71.2\%    & 75.2\%    & 76.9\%    \\
\hspace{5mm}GBM                                                & 73.0\%     & 72.7\%    & 75.5\%    & 77.2\%    \\
\textbf{Proposed Methods}                          &            &           &           &           \\
\hspace{5mm}Logit with Gaussian Gradient                      & \underline{69.3\%}     & 68.3\%    & \underline{73.6\%}    & 74.7\%    \\
\hspace{5mm}Logit with Gaussian Gradient, PoiBin Backtracking & 69.3\%     & \underline{68.4\%}    & 73.6\%    & 74.7\%    \\
\hspace{5mm}Logit with Gaussian Gradient, PoiBin Backtracking, \\ \hspace{15mm} True Gradient     & 69.3\%     & 68.4\%    & 72.2\%    & 74.5\%    \\
\hspace{5mm}Neural Net with Gaussian Gradient                  & 48.3\%     & 46.9\%    & 50.3\%    & 53.5\%    \\
\textbf{Comparison Methods}                        &            &           &           &           \\
\hspace{5mm}Logistic Regression on Aggregates                  & 65.9\%     & 60.0\%    & 71.7\%    & 69.0\%    \\
\hspace{5mm}Ecological Regression                              & 67.6\%     & 66.7\%    & 72.8\%    &\underline{75.0\%}    \\
\hspace{5mm}Inverse Calibration                                & 61.1\%     & 61.9\%    & 72.9\%    & 41.5\%    \\
\hspace{5mm}Mean Map                                           & 51.5\%     & 60.1\%    & 33.3\%    & 31.9\%    \\
\hspace{5mm}Laplacian Mean Map                                 & 51.0\%     & 46.1\%    & 37.9\%    & 51.1\%    \\
\hspace{5mm}Alternating Mean Map                               & 58.9\%     & 62.6\%    & 58.0\%    & 58.4\%   
\end{tabular}}
\end{table}

\begin{table}
\caption{\label{tab:allModel}ROC AUC scores for models predicting voter turnout, fit to demographics and voter history data sets. Highest values among ecological inference models are underlined. }
\centering
\fbox{
\begin{tabular}{lllll}
                                                   & \multicolumn{4}{c}{\textbf{Demographics and Voting History}}  \\
                                                   & \textbf{2017} & \textbf{2016} & \textbf{2015} & \textbf{2014} \\
\textbf{Standard Methods (non-ecological)}                  &               &               &               &               \\
\hspace{5mm}Logistic Regression                                & 85.9\%        & 84.5\%        & 88.6\%        & 89.5\%        \\
\hspace{5mm}GBM                                                & 86.2\%        & 85.5\%        & 88.8\%        & 89.6\%        \\
\textbf{Proposed Methods}                                   &               &               &               &               \\
\hspace{5mm}Logit with Gaussian Gradient                      & \underline{83.9\%}        & \underline{82.0\%}        & \underline{81.0\%}        & 86.3\%        \\
\hspace{5mm}Logit with Gaussian Gradient, PoiBin Backtracking & 83.8\%        & 82.0\%        & 81.0\%        & \underline{86.4\%}        \\
\hspace{5mm}Logit with Gaussian Gradient, PoiBin Backtracking, \\ \hspace{15mm} True Gradient      & 83.8\%        & 81.9\%        & 80.6\%        & 86.3\%        \\
\hspace{5mm}Neural Net with Gaussian Gradient                  & 72.1\%        & 76.8\%        & 80.4\%        & 74.1\%        \\
\textbf{Comparison Methods}                                 &               &               &               &               \\
\hspace{5mm}Logistic Regression on Aggregates                  & 75.0\%        & 72.4\%        & 77.2\%        & 76.8\%        \\
\hspace{5mm}Ecological Regression                              & 67.5\%        & 68.7\%        & 71.8\%        & 76.1\%        \\
\hspace{5mm}Inverse Calibration                                & 64.2\%        & 77.6\%        & 78.4\%        & 66.9\%        \\
\hspace{5mm}Mean Map                                           & 45.4\%        & 54.4\%        & 48.4\%        & 51.8\%        \\
\hspace{5mm}Laplacian Mean Map                                 & 49.5\%        & 51.5\%        & 57.6\%        & 49.4\%        \\
\hspace{5mm}Alternating Mean Map                               & 51.9\%        & 52.9\%        & 44.4\%        & 46.2\%       
\end{tabular}}
\end{table}
\section{Conclusions and Future Work}\label{sec:conc}

Our results extend the literature on ecological inference in several key ways. We model individual voter probabilities via a logistic regression, and pose the problem as a maximum likelihood estimation, proving new results about the existence of the MLE and the curvature of the log-likelihood. We use an approximate algorithm for fitting this model to real election data, making use of a normal approximation. 

The comparative evaluation reveals that this model outperform other ecological inference techniques. This method is intuitively appealing as it is simple to explain and fast to compute. We have also demonstrated that a neural network model is not competitive in this setting---at least without substantial additional work to prevent the model from overfitting. 

On the algorithmic side, further theoretical work is necessary to understand when a finite MLE exists, and when gradient ascent with a Gaussian approximation will not be sufficient. On applications with smaller data sets, or in which outcome proportions close to 0 or 1 are more frequent, the approximation will likely degrade. In such cases, use of backtracking on the true likelihood or ascent on the true gradient may be necessary (and will be more computationally feasible on small data sets). The possibility of using DC algorithms for model fitting also represents an intriguing option. These algorithms may outperform if they can be prevented from stopping a poor local optima. 

On the modeling side, this model-fitting machinery can be extended to a wide class of parametric models. Simple variants on the logistic regression, such as probit or cauchit models, may perform well. A bigger prize would be an expressive model, able to capture non-linearities and interactions, without a high risk of overfit. This remains a direction for future research.

\emph{This work was supported by the Department of Defense (DoD) through the National Defense Science \& Engineering Graduate Fellowship (NDSEG). We thank Pete Mohanty, Art Owen, and Michael Baiocchi for their guidance and feedback on this work.}

\appendix




\section{Gradient and Hessian}

The likelihood given in Equation \ref{eq:lik} can be expressed as: 
\begin{align*}
\ell(\beta) &= \sum_i \log \left( \sum_{A \in F_{D_i}} \prod_{j \in A} \sigma(\xij^\tran\beta) \prod_{j \in A^c}(1 - \sigma(\xij^\tran\beta)) \right) \\
&= \sum_i \log \left( \sum_{A \in F_{D_i}} \exp \left( \sum_{j \in A} \xij^T \beta \right) \right) - \sum_{j \in S_i} \log \left( 1 + \exp(\xij^T \beta) \right) 
\end{align*}

By direct computation, the gradient with respect to $\beta$ is
\begin{align*}
\nabla_{\beta} \ell(\beta) &=  \sum_i  \sum_{A \in F_{D_i}}   \frac{ \exp \left(\sum_{j \in A} \xij^T \beta \right)}{\sum_{A' \in F_D}  \exp \left(\sum_{j \in A'} \xij^T \beta \right)}\cdot \sum_{j \in A} \xij - \sum_{j \in S_i}\frac{\exp (\xij^T \beta)}{1 + \exp(\xij^T \beta)} \xij \,.
\end{align*}

The Hessian is 
\begin{align*}
\nabla_{\beta} \ell(\beta) &=  \sum_i  \sum_{A \in F_{D_i}}   \frac{ \exp \left(\sum_{j \in A} \xij^T \beta \right)}{\sum_{A' \in F_D}  \exp \left(\sum_{j \in A'} \xij^T \beta \right)}\cdot \left( \sum_{j \in A} \xij \right)\left( \sum_{j \in A} \xij \right)^\tran-\\ &
\left( \sum_{A \in F_{D_i}} \frac{ \exp \left(\sum_{j \in A} \xij^T \beta \right)}{\sum_{A' \in F_D}  \exp \left(\sum_{j \in A'} \xij^T \beta \right)} \cdot \sum_{j \in A} \xij  \right) \left( \sum_{A \in F_{D_i}} \frac{ \exp \left(\sum_{j \in A} \xij^T \beta \right)}{\sum_{A' \in F_D}  \exp \left(\sum_{j \in A'} \xij^T \beta \right)} \cdot \sum_{j \in A} \xij  \right)^{\tran} - \\ & \sum_{j \in S_i}\frac{\exp (\xij^T \beta)}{(1 + \exp(\xij^T \beta))^2} \xij \xij^{\tran} \,.
\end{align*}

\section{Alternative Forms of Gradient and Hessian}\label{app:gradComp}

For some choice of $i$, we consider the quantity
\[ \E \left( \sum_{j \in S_i} \xij Y_{ij} \mid \sum_{j \in S_i} \yij = D_i \right) \,. \] 
This is easily interpreted in the sampling setting: the expected sample sum of the covariate vectors, conditional on the total number of units sampled among those in $S_i$. Careful expansion yields 

\begin{align*}
\E \left( \sum_{j \in S_i} \xij Y_{ij} \mid \sum_{j \in S_i} \yij = D_i \right) &= \sum_{A \in F_{D_i}} P(\text{sample $A$ is drawn} \mid \text{sample of size $D$ is drawn}) \cdot \sum_{j \in A} \xij \\
&= \sum_{A \in F_{D_i}}  \frac{P(\text{sample $A$ is drawn})}{P(\text{sample of size $D$ is drawn})} \cdot \sum_{j \in A} \xij \\
&=  \sum_{A \in F_{D_i}} \frac{\prod_{i \in A} \pij \prod_{i \in A^c} (1 - \pij) }{\sum_{A' \in F_D} \prod_{i \in A'} \pij \prod_{i \in A'^c}(1- \pij)} \cdot \sum_{j \in A} \xij \\
&= \sum_{A \in F_D} \frac{ \exp \left(\sum_{j \in A} \xij^T \beta \right)}{\sum_{A' \in F_D}  \exp \left(\sum_{j \in A'} \xij^T \beta \right)}\cdot \sum_{j \in A} \xij 
\end{align*}
which we recognize as the first term in our gradient from the prior section. Analogous computations apply for the Hessian, giving us the following equivalent forms: 
\begin{align*}
\nabla_{\beta} \ell(\beta) &= \sum_i \E \left( \sum_{j \in S_i} \xij Y_{ij} \mid \sum_{j \in S_i} \yij = D_i \right) - \E \left( \sum_{j \in S_i} \xij Y_{ij}  \right)\\
\nabla_{\beta}^2 \ell(\beta) &=  \sum_i \cov \left( \sum_{j \in S_i} \xij Y_{ij} \mid \sum_{j \in S_i} \yij = D_i \right) - \cov \left( \sum_{j \in S_i} \xij Y_{ij}  \right)\\
\end{align*}

This form is particularly auspicious for computation of the exact gradient, as the first term can be expressed 
\[  \E \left( \sum_{j \in S_i} \xij Y_{ij} \mid \sum_{j \in S_i} \yij = D_i \right) = \sum_{j \in S_i} \frac{P\left(\sum_{k \neq j} \yij = D_i - 1\right)\pij}{P \left( \sum_{k \in S_i} \yij = D_i \right)}  \xij \] 
which we recognize as requiring the evaluation of only $|S_i| + 1$ Poisson binomial probabilities to compute. This is preferable to computations based on the original form of the gradient, which involved a combinatorial sum. 




\section{Alternative Model: Neural Network}\label{sec:neural}
In the neural network model, the Democratic support probability is modeled as 
\begin{align*}
\boldsymbol{h}_{ij} &= \sigma\left(W_1 \xij  + b_1 \right) \\
\pij &= \sigma\left(W_2 \boldsymbol{h}_{ij} + b_2 \right) \, ,
\end{align*}
where $W_1, W_2$ are weight matrices and $b_1, b_2$ bias vectors. This network has a single hidden layer. The dimension of $\boldsymbol{h}_{ij} $ can be selected by the modeler, though we used ten hidden nodes for our experiments. 

In the case of a general neural network model, it is typically believed that the loss is non-convex but that local optima are close to the global minimum in function value \citep{choromanska2015loss}. We assume that local optima also exist for the neural network model in our setting. 

With the Gaussian approximation to the Poisson binomial likelihood, the neural network gradients are 
\begin{align*}
\frac{\partial \ell_i}{\partial \pij} &=  \frac{-1 + 2\pij}{2 \phi_k^2} + \frac{1 - 2\pij}{2 \phi_k^4} (D_k - \mu_k)^2  +  \frac{D_k - \mu_k}{\phi_k^2}\\
\frac{\partial \ell_i}{\partial b_2} &= \sum_{j\in S_i} \frac{\partial \ell_k}{\partial \pij} \pij (1 - \pij) \\
\frac{\partial \ell_i}{\partial W_2} &= \left\{ \boldsymbol{h}_{ij} \right\}_j^T \left\{ \frac{\partial \ell}{\partial \pij} \pij (1 - \pij)\right\}_{j} \\
\frac{\partial \ell_i}{\partial b_1} &= \sum_{j \in S_i} \left\{ \frac{\partial \ell}{\partial \pij} \pij (1 - \pij)\right\}_{j} W_2^T \circ  \left\{ \boldsymbol{h}_{ij}(1 - \boldsymbol{h}_{ij})\right\}_j \\
\frac{\partial\ell_i}{\partial W_1} &= \left\{ \xij \right\}_j^T \left\{ \frac{\partial \ell}{\partial \pij} \cdot \pij (1 - \pij)\right\}_{j} W_2^T \circ  \left\{ \boldsymbol{h}_{ij}(1 - \boldsymbol{h}_{ij})\right\}_j \,,
\end{align*}
where $\{\xij\}_j$ denotes a column vector consisting of the entries $\xij$ for all $j \in S_i$ and $\circ$ denotes a Hadamard product.

{
\bibliographystyle{chicago}
\bibliography{csbib}
}

\end{document}